\documentclass[conference]{IEEEtran}
\IEEEoverridecommandlockouts
\usepackage{amsmath,amsfonts}
\usepackage{array}
\usepackage{textcomp}
\usepackage{stfloats}
\usepackage{url}
\usepackage{amsmath, amssymb, amsthm}
\usepackage{verbatim}
\usepackage{graphicx}
\usepackage{amssymb}
\usepackage{hyperref}
\usepackage{graphicx}
\usepackage{amsmath}
\usepackage{longtable}
\usepackage{algorithm} 
\usepackage{algpseudocode} 
\usepackage{mathrsfs}
\usepackage{subcaption}
\usepackage{mathtools}
\usepackage{pifont}
\usepackage{color}
\usepackage{lineno}
\usepackage{graphicx}  
\usepackage{makecell} 
\usepackage{pdflscape}
\usepackage{adjustbox}
\usepackage[utf8]{inputenc}
\usepackage{tabularx}
\usepackage{blindtext}
\usepackage{longtable}
\usepackage{lscape}
\usepackage{amsthm}
\usepackage{graphicx}
\usepackage{subcaption} 
\usepackage{caption}
\usepackage{xcolor}

\usepackage{setspace}
\usepackage{notoccite} 
\usepackage{lscape} 
\usepackage{mwe}
\usepackage{booktabs}
\usepackage{multirow}   
\usepackage{arydshln}   

\usepackage{amsthm}
\newtheorem{theorem}{Theorem}[section]

\newtheorem{lemma}[theorem]{Lemma}
\theoremstyle{definition}

\theoremstyle{definition}

\newcommand{\RNum}[1]{\lowercase\expandafter{\romannumeral #1\relax}}
\newcommand{\RNumU}[1]{\uppercase\expandafter{\romannumeral #1\relax}}
\usepackage[numbers]{natbib}

\def\BibTeX{{\rm B\kern-.05em{\sc i\kern-.025em b}\kern-.08em
    T\kern-.1667em\lower.7ex\hbox{E}\kern-.125emX}}
\begin{document}

\title{Residual-Guided Randomized Neural Networks%
\thanks{Accepted at the 2026 IEEE World Congress on Computational Intelligence (WCCI 2026).}
}

\author{
\IEEEauthorblockN{Mushir Akhtar}
\IEEEauthorblockA{
\textit{Department of Mathematics} \\
\textit{Indian Institute of Technology Indore}\\
phd2101241004@iiti.ac.in}
\and
\IEEEauthorblockN{M. Tanveer}
\IEEEauthorblockA{
\textit{Department of Mathematics} \\
\textit{Indian Institute of Technology Indore}\\
mtanveer@iiti.ac.in}
\and
\IEEEauthorblockN{Mohd. Arshad}
\IEEEauthorblockA{
\textit{Department of Mathematics} \\
\textit{Indian Institute of Technology Indore}\\
arshad@iiti.ac.in}
}

\maketitle
\begin{abstract}
Randomized neural networks enable fast and analytically tractable training by fixing the input to hidden layer parameters at random and learning the output weights in closed form; however, their performance critically depends on a single uninformed draw of hidden units. This one shot and task uninformed feature construction often leads to redundant representations and suboptimal utilization of model capacity. To address this limitation, we propose a simple and broadly applicable residual guided procedure that greedily constructs the hidden layer using a closed form residual decrease criterion. At each stage, we (i) generate a pool of random candidate units, (ii) score each candidate by the exact reduction it induces in the ridge regularized objective, (iii) select the top $k$ units, and (iv) refit the readout in closed form using the standard design with direct input links. This procedure yields a progressive training process with a guaranteed monotonic decrease of the training objective. The method is model agnostic: only the candidate generation is architecture specific, while the scoring selection refitting loop is shared across models. Extensive experiments on 71 benchmark datasets from the UCI repository, covering both binary and multiclass classification tasks, demonstrate that the proposed residual-guided models consistently outperform their baseline counterparts in terms of accuracy, stability, and overall ranking performance. 

\end{abstract}

\begin{IEEEkeywords}
Randomized neural networks, residual guided learning, closed form training, random feature construction, greedy incremental learning.
\end{IEEEkeywords}

\section{Introduction}
\IEEEPARstart{R}{andomized} neural networks (RaNNs) have attracted significant attention as efficient alternatives to deep learning models, particularly in scenarios involving limited data, strict computational constraints, or real-time requirements \cite{pao1994learning, cao2018review, suganthan2021origins}. Unlike deep architectures trained via iterative backpropagation, RaNNs fix the input-to-hidden layer parameters at random and learn only the output-layer weights, typically using closed-form. This decoupling eliminates gradient-based optimization, substantially reduces training complexity, and yields analytically tractable learning procedures while retaining universal approximation capability under mild conditions \cite{igelnik1995stochastic, needell2024random}.

Within this paradigm, several representative RaNN architectures have been proposed, differing primarily in how random features are constructed and incorporated into the predictive model. The Random Vector Functional Link (RVFL) network \cite{pao1994learning, malik2023random} is a widely adopted formulation that combines nonlinear random hidden features with direct input-to-output links, enabling the preservation of linear input structure alongside nonlinear representations. The Extreme Learning Machine (ELM) \cite{huang2006extreme} can be viewed as a special case of RVFL in which the direct input connections are removed, further simplifying the architecture while maintaining fast closed-form training. To enhance representational diversity without increasing depth, the Broad Learning System (BLS) \cite{chen2017broad1, chen2017broad2} was proposed as an alternative that expands the network horizontally by incrementally adding feature and enhancement nodes, enabling efficient incremental learning without retraining from scratch.

Building upon these foundational architectures, extensive research has focused on improving RVFL, ELM, and BLS models in terms of stability, robustness, and generalization \cite{malik2023random, 9380770}. Representative efforts include architectural refinements \cite{feng2018fuzzy, 9715258}, regularization strategies \cite{jin2018regularized, Akhtar2026}, ensemble formulations \cite{Shi2021, 10552388}, and incremental learning mechanisms \cite{LIU2024110430, 10533441}, many of which report enhanced empirical performance across diverse learning scenarios. Recently, copula-aligned weight initialization (CAWI) has been proposed to incorporate dependency-aware copula modeling into weight initialization, aligning randomized weights with the intrinsic structure of the input data and thereby improving representation quality and learning performance \cite{akhtar2026cawi}. Robust objectives and sample-reweighting mechanisms provide complementary directions: the HawkEye loss combines boundedness, smoothness, and an insensitive zone for robust regression \cite{akhtar2025hawkeye}, while its integration with RVFL yields a noise- and outlier-resistant classifier \cite{akhtar2025advancing}. For class-imbalanced settings, a location-aware slack-factor fuzzy SVM improves the assignment of sample memberships and reduces minority-class misclassification \cite{tanveer2025enhancing}.

Despite this progress, most existing extensions share a common structural limitation: hidden representations are still generated through fixed, randomly sampled feature mappings. Consequently, increasing the number of hidden nodes or feature groups often introduces redundancy, leading to inefficient utilization of model capacity and limited adaptability to the data distribution. In particular, indiscriminate expansion with additional random features provides no guarantee that newly introduced representations contribute meaningfully to reducing the supervised residual error.

These observations motivate the need for randomized learning frameworks that retain closed-form training while enabling principled and data-driven feature expansion. In this work, we propose a residual-guided randomized neural network framework that incrementally constructs hidden representations by explicitly leveraging the current residual of the model. At each stage, randomly generated candidate features are evaluated based on their closed-form contribution to residual reduction, and only the most effective features are retained. This residual-guided strategy enables adaptive and efficient feature construction while preserving the computational efficiency and analytical tractability of randomized neural networks. Importantly, the proposed framework is model-agnostic and can be seamlessly integrated into a broad class of RaNNs, including RVFL, ELM, BLS, and their variants, without modifying their fundamental training procedures.

The main contributions of this work are summarized as follows:

\begin{enumerate}
    \item We propose a general residual-guided learning framework for randomized neural networks that enables adaptive, data-driven hidden feature expansion while retaining closed-form output weight learning.
    \item The proposed framework is model-agnostic and can be seamlessly integrated into a broad class of randomized neural networks, including RVFL, ELM, BLS, and their variants, without altering their core architectures or training paradigms.
    \item We provide a theoretical monotonicity guarantee, establishing that the ridge-regularized training objective is non-increasing with the progressive addition of residual-guided hidden features, which is further validated through empirical analysis.
    \item Extensive experimental evaluations on 71 benchmark datasets covering both binary and multiclass classification tasks demonstrate that the proposed residual-guided models consistently outperform their baseline counterparts.
\end{enumerate}

The remainder of this paper is organized as follows. 
Section~\ref{sec:prelim} presents the preliminaries.
Section~\ref{sec:motivation} provides the motivation and problem statement. Section~\ref{sec:method} introduces the proposed framework. Section~\ref{Experiment-section} reports experimental results. Finally, Section~\ref{Conclusions-section} concludes the paper.

\section{Preliminaries}
\label{sec:prelim}
This section establishes the notation and reviews the Random Vector Functional-Link (RVFL) network, one of the standard randomized neural network.

\subsection{Notation}
\label{subsec:notation}
Let $\mathbf{U}\in\mathbb{R}^{n\times p}$ denote the input data matrix with $n$ samples and $p$ features, and let $\mathbf{T}\in\mathbb{R}^{n\times q}$ denote the corresponding target matrix, where $q$ is the output dimension (for classification, $\mathbf{T}$ is typically one-hot encoded). Randomized hidden-layer activations are collected in $\mathbf{G}\in\mathbb{R}^{n\times h}$, generated via random affine transformations of $\mathbf{U}$ followed by an elementwise nonlinearity $\phi(\cdot)$. Each hidden unit is parameterized by a random weight vector $\mathbf{w}\in\mathbb{R}^{p}$ and a bias $\beta\in\mathbb{R}$, yielding the random weight matrix $\mathbf{W}\in\mathbb{R}^{p\times h}$. The augmented feature matrix is defined as $\mathbf{A}=[\,\mathbf{U}\;\mathbf{G}\,]\in\mathbb{R}^{n\times(p+h)}$. The output-layer parameter matrix is denoted by $\mathbf{\Omega}\in\mathbb{R}^{(p+h)\times q}$; when an explicit bias term is included, $\mathbf{A}$ is augmented with a column of ones. The transpose operator is denoted by $(\cdot)^{\top}$, $\mathbf{e}_n\in\mathbb{R}^{n}$ denotes the vector of ones, and $\mathbf{I}_k$ denotes the $k\times k$ identity matrix.


\subsection{Random Vector Functional-Link Network (RVFL)}
\label{subsec:rvfl}
RVFL is a shallow feedforward architecture consisting of input, hidden, and output layers, distinguished by the presence of direct links from the input layer to the output layer. These direct links enable RVFL to jointly exploit linear structure from the raw inputs and nonlinear structure from randomized hidden features, often improving stability and generalization. The standard RVFL architecture is illustrated in Fig.~\ref{fig:RVFL-architecture}.

In RVFL, the hidden-layer parameters are randomly sampled once and kept fixed throughout training. Given $\mathbf{W}\in\mathbb{R}^{p\times h}$ and bias $\beta\in\mathbb{R}$, the hidden-layer activations for input matrix $\mathbf{U}\in\mathbb{R}^{n\times p}$ are computed as
\[
\mathbf{G} = \phi\!\big(\mathbf{U}\mathbf{W} + \mathbf{e}_n \beta\big)\in\mathbb{R}^{n\times h}.
\]
The output layer operates on the augmented representation
$\mathbf{A}=[\,\mathbf{U}\;\mathbf{G}\,]\in\mathbb{R}^{n\times(p+h)}$
and learns the readout matrix $\mathbf{\Omega}$ by solving the
ridge-regularized least-squares problem
\begin{equation}
\label{eq:rvfl-objective}
\mathbf{\Omega}^\star
=
\arg\min_{\mathbf{\Omega}}
\;\|\mathbf{A}\mathbf{\Omega}-\mathbf{T}\|_F^2
+\lambda\|\mathbf{\Omega}\|_F^2,
\qquad \lambda>0.
\end{equation}
This problem admits a closed-form solution, which can be expressed
in either primal or dual form depending on the relative dimensions
of $\mathbf{A}$,
\begin{equation}
\label{eq:rvfl-solution}
\mathbf{\Omega}^\star=
\begin{cases}
(\mathbf{A}^{\top}\mathbf{A}+\lambda\mathbf{I}_{p+h})^{-1}
\mathbf{A}^{\top}\mathbf{T},
& \text{if } (p+h)\le n,\\[4pt]
\mathbf{A}^{\top}
(\mathbf{A}\mathbf{A}^{\top}+\lambda\mathbf{I}_n)^{-1}
\mathbf{T},
& \text{otherwise}.
\end{cases}
\end{equation}

\begin{figure}[t]
    \centering
    \includegraphics[width=0.9\linewidth]{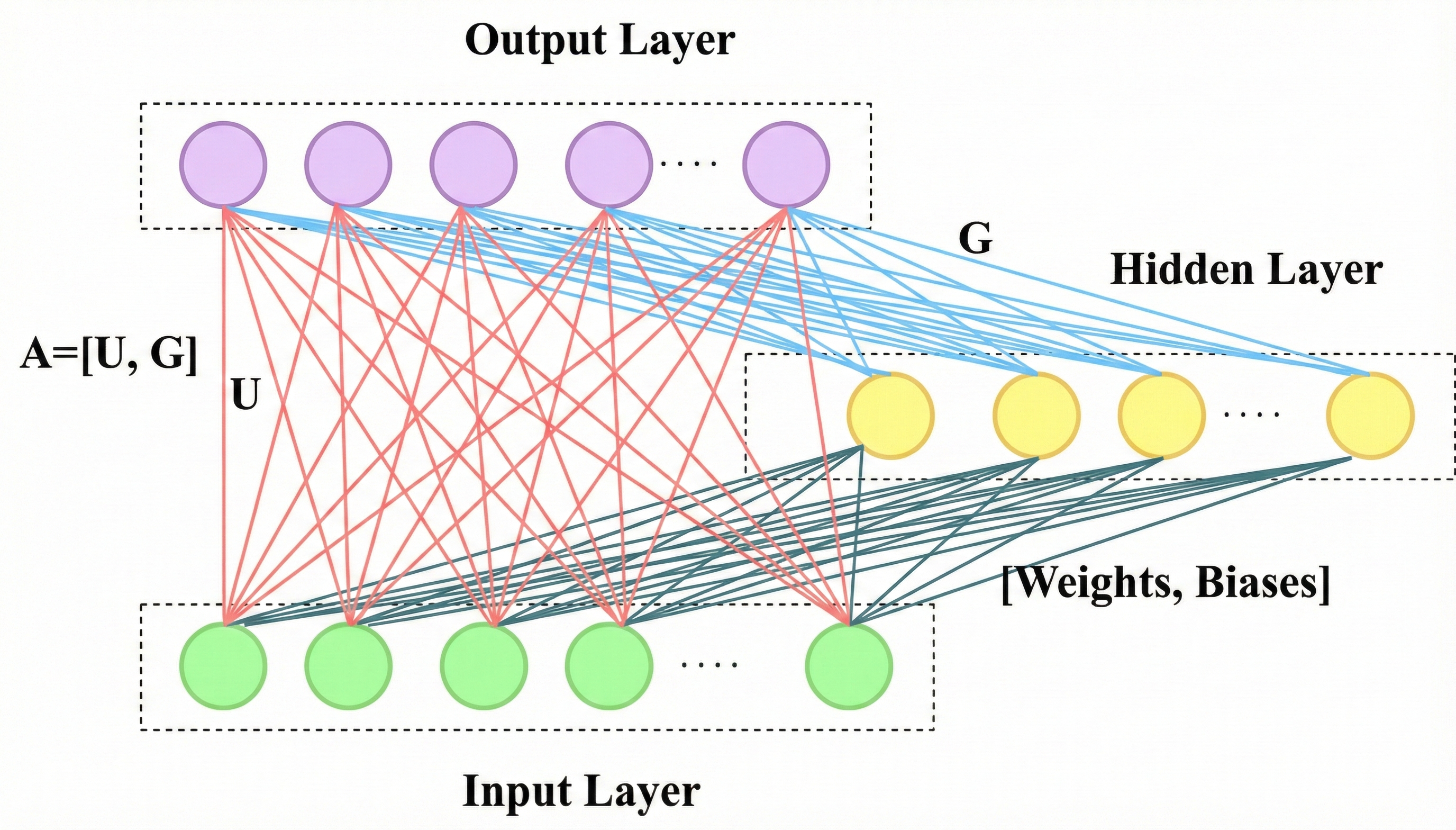}
    \caption{Architecture of a standard RVFL network.}
    \label{fig:RVFL-architecture}
\end{figure}

\section{Motivation \& Problem Statement}
\label{sec:motivation}

RaNNs are attractive since they eliminate backpropagation while enhancing linear models with nonlinear features. However, their performance critically depends on a \emph{single, task-uninformed draw} of randomized hidden features. This section clarifies why such a design is fragile and formulates the precise optimization problem addressed in this work.\\

\noindent
\textbf{Note:} For clarity, we present the subsequent discussion in the RVFL setting, a representative RaNN architecture, and note that the proposed framework extends naturally to other RaNNs.\\

\noindent
\textbf{Setting:}
Let $\mathbf{U}\in\mathbb{R}^{n\times p}$ denote the input data matrix and
$\mathbf{T}\in\mathbb{R}^{n\times q}$ the corresponding target matrix.
A randomized mapping generates hidden-layer activations
$\mathbf{G}\in\mathbb{R}^{n\times h}$, which are concatenated with the inputs to form the augmented matrix
$\mathbf{A}=[\,\mathbf{U}\;\;\mathbf{G}\,]\in\mathbb{R}^{n\times(p+h)}$.
The output layer uses a linear readout
$\mathbf{\Omega}\in\mathbb{R}^{(p+h)\times q}$, learned via ridge-regularized least squares,
\begin{equation}
\label{eq:ridge-obj-motiv}
\mathcal{J}(\mathbf{G}) \;\triangleq\;
\min_{\mathbf{\Omega}\in\mathbb{R}^{(p+h)\times q}}
\big\|\left[\,\mathbf{U}\;\; \mathbf{G}\,\right]\mathbf{\Omega}
- \mathbf{T}\big\|_F^2
\;+\; \lambda\|\mathbf{\Omega}\|_F^2.
\end{equation}

For any fixed $\mathbf{G}$, the minimizer of \eqref{eq:ridge-obj-motiv} is available in closed form. Consequently, the \emph{effective design problem} is to construct $\mathbf{G}$ so as to minimize $\mathcal{J}(\mathbf{G})$ under a prescribed hidden-unit budget.

\subsection{Core Limitations}
In the standard RVFL formulation, $\mathbf{G}$ is generated once by sampling a random affine map and applying a nonlinearity elementwise. This one-shot, label-uninformed construction leads to several practical shortcomings:
\begin{enumerate}
\item \textbf{Feature misalignment and inefficiency.}
The randomly generated columns of $\mathbf{G}$ are not tailored to the supervised task and need not align with discriminative directions in the data. Consequently, many hidden units may exhibit weak correlation with the \emph{supervised residual} after fitting the linear readout, leading to inefficient use of model capacity.

\item \textbf{Lack of adaptive feature selection.}
Once $\mathbf{G}$ is fixed, ridge regularization can only uniformly shrink the output weights; it cannot selectively suppress or discard uninformative or noisy hidden features. This restriction yields a suboptimal bias-variance trade-off and can degrade generalization when $h$ is constrained.

\item \textbf{Sensitivity to random initialization and lack of progressive guarantees.}
Performance is highly sensitive to the random seed, as the entire representation depends on a single draw of $\mathbf{G}$.
Moreover, there is no principled mechanism for progressive model building: intermediate subsets of hidden units are not guaranteed to yield monotonic improvement in the training objective or predictive performance as additional hidden units are added.
\end{enumerate}

\subsection{Problem Statement}
Let $\mathcal{S}$ denote the set of admissible hidden features generated by the randomized mapping,
\begin{equation}
\label{eq:feature-set}
\mathcal{S}
\;=\;
\left\{
\mathbf{g}\in\mathbb{R}^{n}
~\middle|~
\mathbf{g}
=
\phi\!\left(\mathbf{U}\mathbf{w}
+ \mathbf{e}_n \beta\right),
\;
(\mathbf{w},\beta)\sim \mathcal{P}
\right\},
\end{equation}

where $\mathcal{P}$ denotes a prescribed sampling distribution over
$\mathbb{R}^{p}\times\mathbb{R}$.

Given a hidden-unit budget $h\in\mathbb{N}$ and a regularization parameter $\lambda>0$, the objective is to select a collection of $h$ hidden features from $\mathcal{S}$ that minimizes the ridge-regularized training loss. Formally, the design problem is
\begin{equation}
\label{eq:subset-problem}
\min_{\mathbf{G}=[\,\mathbf{g}_1~\cdots~\mathbf{g}_h\,]}
\;\;
\mathcal{J}(\mathbf{G})
\quad
\text{s.t.}\quad
\mathbf{g}_j\in\mathcal{S}\;\;\forall j,
\;\;
\mathbf{G}\in\mathbb{R}^{n\times h}.
\end{equation}

Equivalently, \eqref{eq:subset-problem} seeks a size-$r$ multiset of hidden features drawn from $\mathcal{S}$ such that, when concatenated with the direct input links $\mathbf{U}$, the resulting augmented representation minimizes the ridge objective in \eqref{eq:ridge-obj-motiv}. This formulation defines a \emph{combinatorial feature selection problem} over a large and typically continuous set of admissible features, rendering exact optimization computationally intractable.\\

\noindent
\textbf{Desired Properties:}
A practical alternative to one-shot RVFL should:
\begin{enumerate}
\item preserve \emph{closed-form} training at the output layer (no backpropagation);
\item \emph{adapt} hidden-unit selection to the supervised signal by prioritizing features that most reduce the current objective;
\item admit an \emph{anytime} training procedure with a \emph{monotonic decrease} of the objective as hidden nodes are added.
\end{enumerate}

\noindent
\textbf{Our Approach:}
We address \eqref{eq:subset-problem} via a \emph{stagewise, residual-guided} construction.
At each iteration, a manageable candidate pool is sampled from $\mathcal{S}$.
Each candidate is scored by its \emph{closed-form reduction} in \eqref{eq:ridge-obj-motiv} relative to the current model, the most effective units are selected, and the readout $\mathbf{\Omega}$ is refit exactly on the updated augmented matrix.
This procedure directly targets the limitations of standard RVFL while retaining its hallmark efficiency and closed-form training.

\section{Method}
\label{sec:method}
In this section, we propose a residual-guided construction for randomized neural networks that retains the key advantages of fixed randomized hidden mappings and closed-form training while overcoming the limitations of one-shot hidden-layer sampling. The core idea is to replace the single, task-uninformed draw of hidden features with a supervised, stagewise selection procedure guided by the current residual error.

\subsection{Training Objective}
\label{subsec:objective}

Given an input matrix $\mathbf{U}\in\mathbb{R}^{n\times p}$ and a hidden-feature matrix
$\mathbf{G}\in\mathbb{R}^{n\times h}$, RVFL forms the augmented representation
$\mathbf{A}=[\,\mathbf{U}\;\;\mathbf{G}\,]$ and learns a linear readout
$\mathbf{\Omega}\in\mathbb{R}^{(p+h)\times q}$ via ridge-regularized least squares,
\begin{equation}
\label{eq:method-ridge}
\mathbf{\Omega}^\star(\mathbf{A},\lambda)
=
\arg\min_{\mathbf{\Omega}}
\|\mathbf{A}\mathbf{\Omega}-\mathbf{T}\|_F^2
+
\lambda\|\mathbf{\Omega}\|_F^2,
\qquad \lambda>0.
\end{equation}

The resulting training objective value is
\begin{equation}
\label{eq:L-of-G}
\mathcal{J}(\mathbf{G})
=
\min_{\mathbf{\Omega}}
\|\,[\,\mathbf{U}\;\;\mathbf{G}\,]\mathbf{\Omega}-\mathbf{T}\|_F^2
+
\lambda\|\mathbf{\Omega}\|_F^2.
\end{equation}
Since the minimizer in \eqref{eq:method-ridge} is available in closed form, learning in RVFL reduces to constructing a hidden-feature matrix $\mathbf{G}$, under a prescribed hidden-unit budget $h$, that yields a small value of $\mathcal{J}(\mathbf{G})$. This observation motivates a direct focus on how hidden features are generated and selected.



\subsection{Residual-Guided Stagewise Construction}
\label{subsec:greedy}

Rather than generating all $h$ hidden features in a single step, the proposed method constructs the hidden-feature matrix $\mathbf{G}$ incrementally in a stagewise manner.

Let $\mathbf{G}_t\in\mathbb{R}^{n\times h_t}$ denote the hidden-feature matrix after stage $t$, with $h_0=0$, and define the corresponding augmented representation as
$\mathbf{A}_t=[\,\mathbf{U}\;\;\mathbf{G}_t\,]$. At each stage, the current model is first fitted exactly by solving the ridge regression problem, yielding the optimal readout and the associated residual
\begin{equation}
\label{eq:residual}
\mathbf{\Omega}_t=\mathbf{\Omega}^\star(\mathbf{A}_t,\lambda),
\qquad
\mathbf{R}_t= \mathbf{A}_t\mathbf{\Omega}_t - \mathbf{T}.
\end{equation}
The residual $\mathbf{R}_t$ represents the portion of the supervised signal that is not yet captured by the current representation. Since the training objective $\mathcal{J}(\mathbf{G}_t)$ depends directly on this residual, reducing $\|\mathbf{R}_t\|_F$ leads to a decrease in the objective value.

Each stage of the algorithm then proceeds by (i) generating a temporary pool of candidate random hidden features, (ii) evaluating each candidate according to its ability to reduce the current residual, and (iii) selecting a small subset of the most effective features to be permanently added to the model. After feature selection, the readout is refitted exactly using the augmented representation, and the process is repeated until the hidden-unit budget $h$ is reached or a stopping criterion is satisfied. The individual steps of candidate generation, residual-based scoring, and selection are detailed below.

\subsubsection{Candidate generation}

At stage $t$, a temporary pool of $M$ candidate hidden features is generated. Specifically, random affine parameters are drawn from the sampling distribution
$\mathcal{P}$ and transformed using the activation function $\phi(\cdot)$,
\begin{equation}
\label{eq:candidates}
\mathbf{g}^{(j)}
=
\phi(\mathbf{U}\mathbf{w}^{(j)}+\mathbf{e}_n\beta^{(j)}),
\quad
(\mathbf{w}^{(j)},\beta^{(j)})\sim\mathcal{P},
\quad j=1,\ldots,M.
\end{equation}
The resulting vectors $\{\mathbf{g}^{(j)}\}_{j=1}^M$ form a candidate pool that is used only at the current stage, from which a small number of informative features will be selected and permanently added to the model.

\subsubsection{Residual-based scoring}

Each candidate hidden feature is evaluated based on its ability to reduce the
current training objective.
For a candidate $\mathbf{g}$, we consider augmenting the current representation
$\mathbf{A}_t$ with this feature and optimizing only its associated coefficient,
while keeping the previously learned readout parameters fixed.
Under this setting, the optimal coefficient and the resulting \emph{exact decrease}
in the ridge-regularized objective are given by
\begin{equation}
\label{eq:delta}
\mathbf{\Omega}_{\mathbf{g}}^\star
=
\frac{\mathbf{g}^{\top}\mathbf{R}_t}{\|\mathbf{g}\|_2^2+\lambda},
\qquad
\Delta(\mathbf{g})
=
\frac{\|\mathbf{g}^{\top}\mathbf{R}_t\|_F^2}
{\|\mathbf{g}\|_2^2+\lambda}.
\end{equation}
The score $\Delta(\mathbf{g})$ therefore quantifies the immediate reduction in the
training objective that would be achieved by adding $\mathbf{g}$ alone, making it
a supervised and scale-aware criterion for greedy feature selection.

\subsubsection{Selection and refitting}

After scoring all $M$ candidates, we select the top-$k$ features with the largest
values of $\Delta(\mathbf{g})$ and append them to the current hidden-feature matrix,
\[
\mathbf{G}_{t+1}=[\,\mathbf{G}_t\;\;\mathbf{G}_t^{\mathrm{sel}}\,],
\qquad
\mathbf{A}_{t+1}=[\,\mathbf{U}\;\;\mathbf{G}_{t+1}\,].
\]
The linear readout is then refitted exactly on the updated representation by solving
\eqref{eq:method-ridge}, yielding $\mathbf{\Omega}_{t+1}$.
This stagewise process of candidate generation, residual-based scoring, selection,
and refitting is repeated until the hidden-unit budget $h$ is reached or a validation
criterion terminates training.

The aforementioned steps complete the description of the proposed residual-guided stagewise construction. 
For clarity and completeness, the entire procedure is summarized in Algorithm~\ref{alg:rgrvfl}.
\begin{algorithm}[t]
\caption{Residual-Guided Stagewise Algorithm}
\label{alg:rgrvfl}
\begin{algorithmic}[1]
\Require Data $\mathbf{U}\in\mathbb{R}^{n\times p}$, targets $\mathbf{T}\in\mathbb{R}^{n\times q}$; hidden nodes budget $h$; pool size $M$; block size $k$; regularization $\lambda>0$; sampling distribution $\mathcal{P}$; activation $\phi(\cdot)$.
\Ensure Hidden features $\mathbf{G}\in\mathbb{R}^{n\times h'}$ with $h'\le h$ and readout $\mathbf{\Omega}\in\mathbb{R}^{(p+h')\times q}$.

\State Initialize $\mathbf{G}_0 \gets [\,]$, $h_0 \gets 0$, $t\gets 0$.
\While{$h_t < h$ \textbf{and} stopping criterion not met}
    \State $\mathbf{A}_t \gets [\,\mathbf{U}\;\;\mathbf{G}_t\,]$.
    \State $\mathbf{\Omega}_t \gets \mathbf{\Omega}^\star(\mathbf{A}_t,\lambda)$ \Comment{closed-form ridge solution}
    \State $\mathbf{R}_t \gets \mathbf{A}_t\mathbf{\Omega}_t - \mathbf{T}$.
    \For{$j=1,\dots,M$} \Comment{temporary candidate pool}
        \State Sample $(\mathbf{w}^{(j)},\beta^{(j)}) \sim \mathcal{P}$.
        \State $\mathbf{g}^{(j)} \gets \phi(\mathbf{U}\mathbf{w}^{(j)}+\mathbf{e}_n\beta^{(j)}) \in\mathbb{R}^{n}$.
        \State $s^{(j)} \gets \dfrac{\|\mathbf{g}^{(j)\top}\mathbf{R}_t\|_F^2}{\|\mathbf{g}^{(j)}\|_2^2+\lambda}$ \Comment{score $\Delta(\mathbf{g}^{(j)})$}
    \EndFor
    \State Let $\mathcal{I}$ be indices of the top-$k$ scores $\{s^{(j)}\}_{j=1}^M$ (or fewer if $h_t+k>h$).
    \State $\mathbf{G}_t^{\mathrm{sel}} \gets [\,\mathbf{g}^{(j)}\,]_{j\in\mathcal{I}}$.
    \State $\mathbf{G}_{t+1} \gets [\,\mathbf{G}_t\;\;\mathbf{G}_t^{\mathrm{sel}}\,]$, \quad $h_{t+1}\gets h_t + |\mathcal{I}|$.
    \State $t\gets t+1$.
\EndWhile
\State $\mathbf{A} \gets [\,\mathbf{U}\;\;\mathbf{G}_t\,]$, \quad $\mathbf{\Omega}\gets \mathbf{\Omega}^\star(\mathbf{A},\lambda)$.
\State \Return $\mathbf{G}_t,\mathbf{\Omega}$.
\end{algorithmic}
\end{algorithm}

\subsection{Monotonicity Guarantee}
\label{subsec:guarantee}

In this subsection, we establish a key theoretical property of the proposed residual-guided stagewise method. Specifically, we show that the ridge-regularized training objective is non-increasing across stages, ensuring that the algorithm makes consistent progress as additional hidden features are incorporated.

\begin{lemma}
Let $\mathcal{J}(\mathbf{G})$ be defined as in \eqref{eq:L-of-G}.  
At each stage of the residual-guided procedure, the following inequality holds:
\[
\mathcal{J}(\mathbf{G}_{t+1}) \le \mathcal{J}(\mathbf{G}_t).
\]
\end{lemma}

\begin{proof}
Let $\mathbf{A}_t=[\,\mathbf{U}\;\;\mathbf{G}_t\,]$ and
$\mathbf{A}_{t+1}=[\,\mathbf{U}\;\;\mathbf{G}_{t+1}\,]$, where
$\mathbf{G}_{t+1}=[\,\mathbf{G}_t\;\;\mathbf{G}_t^{\mathrm{sel}}\,]$.
By construction, $\mathbf{A}_t$ is a column-submatrix of $\mathbf{A}_{t+1}$.

Let $\mathbf{\Omega}_t^\star$ denote an optimal solution attaining
$\mathcal{J}(\mathbf{G}_t)$ in \eqref{eq:L-of-G}.
We construct a coefficient matrix $\widetilde{\mathbf{\Omega}}$ for
$\mathbf{A}_{t+1}$ by assigning $\mathbf{\Omega}_t^\star$ to the columns
corresponding to $\mathbf{A}_t$ and setting the coefficients of the newly added
columns to zero.
Then,
\[
\mathbf{A}_{t+1}\widetilde{\mathbf{\Omega}}
=
\mathbf{A}_t\mathbf{\Omega}_t^\star,
\qquad
\|\widetilde{\mathbf{\Omega}}\|_F
=
\|\mathbf{\Omega}_t^\star\|_F.
\]
Consequently, $\widetilde{\mathbf{\Omega}}$ achieves the same objective value
on $\mathbf{A}_{t+1}$ as $\mathbf{\Omega}_t^\star$ does on $\mathbf{A}_t$.
Since $\mathcal{J}(\mathbf{G}_{t+1})$ is defined as the minimum of the
ridge-regularized objective over all coefficient matrices associated with
$\mathbf{A}_{t+1}$, it cannot exceed this value.
Therefore,
\[
\mathcal{J}(\mathbf{G}_{t+1}) \le \mathcal{J}(\mathbf{G}_t),
\]
which completes the proof.
\end{proof}

\subsection{Computational Complexity}
\label{sec:complexity}

At each stage, generating a pool of $M$ candidate hidden features requires
$O(npM)$ operations, corresponding to the evaluation of random affine mappings
and activation functions. Scoring the candidates incurs an additional cost of $O(nMq)$ to compute the
inner products $\mathbf{G}_t^{\mathrm{pool}\top}\mathbf{R}_t$, along with $O(nM)$ operations to evaluate the associated $\ell_2$ norms. Refitting the linear readout after stage $t$ using a naive approach costs
$O(n(p+h_t)^2)$ when using the primal formulation, or $O(n^2(p+h_t))$ when using
the dual formulation. However, since only a small block of $k$ hidden features is added at each stage,
the readout update can be performed incrementally by exploiting the resulting
low-rank modification of the augmented matrix.
This reduces the refitting cost to
$O\bigl(nk(p+h_t)\bigr) + O\bigl(k^2(p+h_t)\bigr).$
Let $T$ denote the total number of stages.
Since $T \approx \lceil h/k \rceil$, the overall computational cost of the
proposed method grows approximately linearly with the hidden-unit budget $h$
for fixed pool size $M$ and block size $k$.

\section{Experimental Results}
\label{Experiment-section}
This section evaluates the proposed residual-guided framework by integrating it into three RaNN architectures, namely RVFL, ELM, and BLS, yielding the corresponding residual-guided models RG-RVFL, RG-ELM, and RG-BLS. The proposed models are compared with their baseline counterparts to assess the effectiveness of residual-guided construction. Performance is evaluated using classification accuracy, average rank, and standard deviation. Experiments are conducted on a total of 71 benchmark datasets, comprising both binary and multiclass classification problems, downloaded from the UCI repository \cite{dua2017uci}. The detailed experimental setup along with the hyperparametr configuration is provided in Section S.I of the supplementary file.

\subsection{Performance Analysis}
Table~\ref{tab:final_summary} summarizes the average classification accuracy, standard deviation, and average rank of the baseline and residual-guided models over 32 binary and 39 multiclass datasets. Overall, the results consistently demonstrate the effectiveness of the proposed residual-guided learning framework across all three randomized architectures, namely RVFL, ELM, and BLS.

For binary classification tasks, the residual-guided variants achieve clear and consistent accuracy gains over their corresponding baselines. In particular, RG-RVFL, RG-ELM, and RG-BLS improve the average accuracy from 80.92\% to 82.92\%, from 80.82\% to 82.76\%, and from 81.65\% to 82.78\%, respectively. Similar trends are observed for multiclass datasets, where the proposed models yield notable improvements, with average accuracy increasing from 73.33\% to 75.94\% for RVFL, from 73.34\% to 75.35\% for ELM, and from 74.96\% to 76.20\% for BLS. These results indicate that residual-guided construction consistently enhances the representational capacity of the underlying randomized learners. Beyond accuracy, the proposed framework also exhibits improved stability. As reflected by the lower average standard deviation values, all residual-guided models demonstrate reduced performance variability across datasets. For binary problems, the standard deviation decreases from 8.91 to 7.88 for RVFL, from 8.36 to 7.76 for ELM, and from 6.24 to 5.75 for BLS. A similar reduction is observed in multiclass settings, confirming that the residual-guided strategy not only improves mean performance but also yields more reliable and robust predictions across diverse data distributions. The superiority of the proposed approach is further reinforced by the average rank analysis. Across both binary and multiclass tasks, the residual-guided variants consistently achieve near-optimal average ranks, all close to one, indicating dominant performance among the compared methods. In contrast, the baseline models exhibit noticeably higher ranks, reflecting inferior relative performance. This consistent ranking advantage highlights the effectiveness of the proposed residual-guided mechanism independent of the specific randomized architecture employed.

Collectively, these results demonstrate that incorporating residual-guided learning into RVFL, ELM, and BLS models leads to systematic improvements in accuracy, stability, and overall ranking performance. Importantly, the gains are observed across both binary and multiclass classification tasks, underscoring the generality of the proposed framework. Detailed dataset-wise results for all evaluated models are provided in Tables~S.I and~S.II of the supplementary material.

To empirically validate the monotonicity property established in Section~\ref{subsec:guarantee}, we analyze the stagewise behavior of the proposed residual-guided framework. The detailed analysis is provided in Section~S.II of the supplementary.

\begin{table}[t]
\centering
\caption{Average accuracy (\%), average standard deviation, and average rank of baseline and residual-guided models on 32 binary and 39 multiclass datasets.
}
\label{tab:final_summary}
\resizebox{\columnwidth}{!}{
\begin{tabular}{c c cc : cc : cc}
\toprule
\multirow{2}{*}{\textbf{Task}} 
& \multirow{2}{*}{\textbf{Metric}} 
& \multicolumn{6}{c}{\textbf{Method Comparison}} \\
\cmidrule(lr){3-8}
& 
& \textbf{RVFL} & \textbf{RG-RVFL$^{\dagger}$} 
& \textbf{ELM} & \textbf{RG-ELM$^{\dagger}$} 
& \textbf{BLS} & \textbf{RG-BLS$^{\dagger}$} \\
\midrule
\multirow{3}{*}{\textbf{Binary (32)}} 
& \textbf{Accuracy {\color{green}$\uparrow$}} 
& 80.92 & \textbf{82.92} 
& 80.82 & \textbf{82.76} 
& 81.65 & \textbf{82.78} \\

& \textbf{Std. Dev. {\color{blue}$\downarrow$}} 
& 8.91 & \textbf{7.88} 
& 8.36 & \textbf{7.76} 
& 6.24 & \textbf{5.75} \\

& \textbf{Rank {\color{red}$\downarrow$}} 
& 1.97 & \textbf{1.03} 
& 1.88 & \textbf{1.13} 
& 1.91 & \textbf{1.09} \\
\midrule
\multirow{3}{*}{\textbf{Multiclass (39)}} 
& \textbf{Accuracy {\color{green}$\uparrow$}} 
& 73.33 & \textbf{75.94} 
& 73.34 & \textbf{75.35} 
& 74.96 & \textbf{76.20} \\

& \textbf{Std. Dev. {\color{blue}$\downarrow$}} 
& 10.43 & \textbf{9.11} 
& 10.12 & \textbf{8.99} 
& 7.98 & \textbf{7.17} \\

& \textbf{Rank {\color{red}$\downarrow$}} 
& 1.96 & \textbf{1.04} 
& 1.94 & \textbf{1.06} 
& 1.86 & \textbf{1.14} \\
\bottomrule
\multicolumn{8}{l}{$^{\dagger}$ denotes the proposed counterpart.}
\end{tabular}
}
\end{table}

\section{Conclusions}\label{Conclusions-section}
In this work, we introduced a residual-guided learning framework for randomized neural networks that addresses a fundamental limitation of existing RaNN architectures, namely the reliance on a single, task-uninformed draw of hidden features. By replacing one-shot random feature construction with a supervised, stagewise selection mechanism guided by the current residual, the proposed framework enables adaptive and efficient hidden-layer expansion while preserving the hallmark advantages of randomized learning, including closed-form training and computational efficiency. We establish a monotonicity guarantee showing that the ridge-regularized training objective is non-increasing as residual-guided hidden features are progressively added. This property provides a foundation for progressive model construction and stands in contrast to conventional randomized networks, where intermediate representations offer no such guarantees. Extensive experimental results on 71 benchmark datasets from the UCI repository, spanning both binary and multiclass classification tasks, confirm the practical effectiveness of the proposed framework. The results highlight the generality of the framework and its ability to enhance diverse randomized models without altering their core training paradigms. In future, one can extend the proposed framework for regression, multilabel, and semi-supervised tasks.

\bibliographystyle{IEEEtranN}
\bibliography{refs.bib}

\clearpage
\setcounter{section}{0}
\setcounter{subsection}{0}
\setcounter{table}{0}
\setcounter{figure}{0}
\setcounter{equation}{0}
\renewcommand{\thesection}{S.\Roman{section}}
\renewcommand{\thesubsection}{\thesection.\Alph{subsection}}
\renewcommand{\thetable}{S.\Roman{table}}
\renewcommand{\thefigure}{S.\arabic{figure}}
\renewcommand{\theequation}{S.\arabic{equation}}
\renewcommand{\theHsection}{supp.\Roman{section}}
\renewcommand{\theHtable}{supp.\Roman{table}}
\renewcommand{\theHfigure}{supp.\arabic{figure}}
\renewcommand{\theHequation}{supp.\arabic{equation}}
\twocolumn[
\begin{center}
{\LARGE\bfseries Supplementary Material}\\
{\large Residual-Guided Randomized Neural Networks}
\end{center}
\vspace{1ex}
]



\section{Experimental Setup and Hyperparameter Configuration}
\label{sec:exp_setup}

All experiments are conducted using MATLAB R2023a on a system equipped with an 11th-generation Intel Core i7-11700 processor running at 2.50~GHz, 16~GB of RAM, and a Windows~11 operating system. To ensure a fair and statistically reliable evaluation, all models are assessed using a 5-fold cross-validation protocol combined with grid-based hyperparameter search. For each dataset, the samples are partitioned into five mutually exclusive folds. For a given hyperparameter configuration, the model is trained on four folds and evaluated on the remaining fold. This process is repeated until each fold has served once as the test set. The classification accuracy obtained across the five folds is averaged, and the configuration achieving the highest mean accuracy is reported as the final performance. In addition to the mean accuracy, the standard deviation across the five folds is also reported to reflect performance stability.

For RVFL, ELM, and their residual-guided counterparts, the regularization parameter $\lambda$ is selected from the set
$
\{10^{-5}, 10^{-4}, \ldots, 10^{5}\},
$
and the number of hidden nodes is tuned within the range
$
5{:}10{:}205.
$
For BLS and its residual-guided variant, the regularization parameter $\lambda$ is chosen from
$
\{10^{-6}, 10^{-4}, \ldots, 10^{6}\}.
$
The number of feature node windows is selected from
$1{:}2{:}21$,
the number of feature nodes per window from
$5{:}5{:}50$,
and the number of enhancement nodes from
$5{:}10{:}105$.
For the proposed residual-guided framework, additional hyperparameters control the stagewise feature selection process. In particular, the block size $k$ denotes the number of hidden features added at each stage, while the pool size $M$ specifies the number of candidate random hidden nodes generated at that stage. The choice of $k$ and $M$ is adapted based on the number of training samples $n$ and the input dimensionality $p$. When $n \leq 2000$ or $p \leq 50$, we consider $k \in \{1,2\}$ and $M \in \{50,100\}$. For intermediate-scale settings where $2000 < n \leq 20000$ or $50 < p \leq 300$, the search space is expanded to $k \in \{1,2,4\}$ and $M \in \{100,200\}$. For larger-scale problems, we use $k \in \{2,4\}$ and $M \in \{100,200\}$.
In addition, the total number of hidden features is constrained by a predefined budget, which is set as
$\max\!\left(50,\; \min\!\left(1000,\; \left\lfloor 0.2\, n \right\rfloor,\; 5p\right)\right)$.
This constraint ensures a balance between model expressiveness and computational efficiency while remaining consistent with the stagewise feature expansion procedure.


\begin{table*}[t]
\centering
\renewcommand{\arraystretch}{1.15}
\caption{Dataset-wise results on 32 binary datasets.}
\label{tab:binary32_acc_rank}
\resizebox{15cm}{!}{
\begin{tabular}{lcccccc}
\toprule
\textbf{Dataset} & \textbf{RVFL} & \textbf{RG-RVFL$^{\dagger}$} & \textbf{ELM} & \textbf{RG-ELM$^{\dagger}$} & \textbf{BLS} & \textbf{RG-BLS$^{\dagger}$} \\
\midrule
acute\_inflammation                 & 100      & 100      & 100      & 100      & 100      & 100 \\
acute\_nephritis                     & 100      & 100      & 100      & 100      & 100      & 100 \\
molec\_biol\_promoter               & 71.8182  & 75.5411  & 68.9610  & 77.3593  & 83.9827  & 85.8874 \\
parkinsons                           & 81.0256  & 82.0513  & 80.5128  & 81.0256  & 81.0256  & 82.0256 \\
pittsburg\_bridges\_T\_OR\_D        & 87.2381  & 89.2381  & 87.1429  & 88.1429  & 89.2381  & 90.1905 \\
bank                                 & 89.6925  & 90.0022  & 89.5378  & 89.8032  & 89.7368  & 89.7589 \\
blood                                & 76.6380  & 76.9047  & 76.6389  & 77.0398  & 77.3047  & 78.1047 \\
breast\_cancer                       & 70.1754  & 70.5263  & 70.1754  & 70.1754  & 70.9074  & 71.5910 \\
breast\_cancer\_wisc                 & 88.4173  & 90.1357  & 88.2785  & 89.5642  & 87.9908  & 88.9928 \\
breast\_cancer\_wisc\_diag           & 93.8457  & 95.2554  & 94.5552  & 95.0784  & 94.2028  & 95.3240 \\
chess\_krvkp                         & 81.1655  & 88.8291  & 79.4141  & 87.9228  & 84.3870  & 83.7923 \\
congressional\_voting                & 63.2184  & 63.6782  & 63.2184  & 63.4483  & 63.4483  & 63.9080 \\
cylinder\_bands                      & 65.6349  & 68.7626  & 68.1915  & 69.1357  & 69.5393  & 70.1275 \\
echocardiogram                       & 84.6724  & 85.4416  & 84.6724  & 86.2108  & 83.1624  & 85.1624 \\
heart\_hungarian                     & 72.4547  & 76.1777  & 76.8615  & 76.1835  & 78.5973  & 77.9018 \\
hepatitis                            & 85.1613  & 85.8065  & 84.5161  & 85.8065  & 87.0968  & 88.2258 \\
horse\_colic                         & 85.5942  & 86.6864  & 85.6016  & 87.2195  & 85.8793  & 86.0648 \\
ionosphere                           & 89.1911  & 92.0362  & 90.3380  & 91.4728  & 88.3461  & 89.1952 \\
mammographic                         & 79.9196  & 80.0248  & 79.6098  & 79.5056  & 79.0895  & 79.1899 \\
monks\_1                             & 84.3292  & 93.1532  & 83.4299  & 91.8951  & 76.0505  & 81.9932 \\
monks\_2                             & 82.3457  & 89.5014  & 82.5193  & 90.8375  & 73.1915  & 74.1942 \\
monks\_3                             & 91.3382  & 92.7846  & 91.5152  & 93.1466  & 86.8272  & 88.9894 \\
musk\_1                              & 67.2763  & 73.1206  & 67.2456  & 71.2368  & 76.2654  & 78.8969 \\
oocytes\_merluccius\_nucleus\_4d    & 82.2898  & 83.9517  & 82.1908  & 84.1502  & 82.7800  & 84.2688 \\
oocytes\_trisopterus\_nucleus\_2f   & 78.8314  & 79.8181  & 77.5206  & 80.0390  & 79.0554  & 80.2728 \\
pima                                 & 73.5736  & 74.4877  & 72.6594  & 74.4843  & 72.5236  & 74.0049 \\
spect                                & 68.3019  & 68.6792  & 67.5472  & 67.9245  & 69.0566  & 71.9434 \\
spectf                               & 79.3431  & 80.4612  & 79.3431  & 79.7135  & 79.7135  & 79.7205 \\
statlog\_australian\_credit          & 68.2609  & 68.5507  & 68.1159  & 68.9855  & 68.4058  & 68.7159 \\
statlog\_german\_credit              & 77.2000  & 77.7000  & 76.4000  & 77.9000  & 76.9000  & 77.9000 \\
statlog\_heart                       & 81.4815  & 81.8519  & 80.7407  & 80.7407  & 82.2222  & 84.2222 \\
tic\_tac\_toe                         & 89.0303  & 92.3653  & 88.7178  & 92.0561  & 96.0143  & 98.3273 \\
\midrule
\textbf{Average Accuracy}            & \textbf{80.92} & \textbf{82.92} & \textbf{80.82} & \textbf{82.76} & \textbf{81.65} & \textbf{82.78} \\
\textbf{Average Std.\ Dev.}          & \textbf{8.91}      & \textbf{7.88}        & \textbf{8.36}        & \textbf{7.76}        & \textbf{6.24}        & \textbf{5.75} \\
\textbf{Average Rank}                & \textbf{1.97}   & \textbf{1.03}     & \textbf{1.88}       & \textbf{1.13}       & \textbf{1.91}     & \textbf{1.09} \\
\bottomrule
\multicolumn{7}{l}{$^{\dagger}$ denotes the proposed counterpart.}
\end{tabular}
}
\end{table*}


\begin{table*}[t]
\centering
\renewcommand{\arraystretch}{1.15}
\caption{Dataset-wise result on 39 multi-class datasets.}
\label{tab:multiclass39_acc_rank}
\resizebox{15cm}{!}{
\begin{tabular}{lcccccc}
\toprule
\textbf{Dataset} & \textbf{RVFL} & \textbf{RG-RVFL$^{\dagger}$} & \textbf{ELM} & \textbf{RG-ELM$^{\dagger}$} & \textbf{BLS} & \textbf{RG-BLS$^{\dagger}$} \\
\midrule
abalone                         & 63.4179 & 63.8732 & 63.3465 & 63.8729 & 63.4657 & 63.9856 \\
annealing                       & 89.1924 & 91.1962 & 89.6356 & 91.0838 & 89.6344 & 90.6400 \\
arrhythmia                      & 65.4823 & 71.0208 & 65.7192 & 68.3663 & 64.1587 & 65.2576 \\
audiology\_std                  & 69.2949 & 69.8718 & 67.7821 & 69.8333 & 69.8333 & 70.8718 \\
cardiotocography\_10clases      & 69.6156 & 71.3099 & 69.2395 & 70.7451 & 65.5708 & 66.4182 \\
cardiotocography\_3clases       & 85.2331 & 86.2683 & 85.3743 & 86.2683 & 85.9390 & 86.8921 \\
contrac                         & 40.5239 & 41.6098 & 40.8666 & 41.2729 & 45.0847 & 45.3691 \\
energy\_y1                      & 89.1902 & 91.5338 & 89.3116 & 90.8777 & 88.4135 & 89.5392 \\
energy\_y2                      & 90.3599 & 91.9277 & 89.9728 & 91.5338 & 90.1070 & 91.3667 \\
flags                           & 52.1188 & 54.1430 & 52.1053 & 53.1309 & 53.1309 & 53.6302 \\
glass                           & 38.1506 & 43.2780 & 39.0808 & 42.8128 & 40.0111 & 40.9524 \\
hayes\_roth                     & 61.2500 & 62.5000 & 60.0000 & 63.1250 & 65.6250 & 67.7500 \\
heart\_cleveland                & 59.6995 & 60.0437 & 58.6721 & 59.0601 & 59.6885 & 59.7104 \\
heart\_switzerland              & 45.6000 & 47.3000 & 47.3000 & 48.9000 & 47.2667 & 45.6667 \\
heart\_va                       & 39.5000 & 41.0000 & 39.0000 & 41.0000 & 41.0000 & 42.0000 \\
image\_segmentation             & 87.8355 & 90.2597 & 86.9697 & 90.3030 & 88.7446 & 90.8596 \\
iris                            & 75.3333 & 76.0000 & 76.0000 & 78.0000 & 77.3333 & 79.3333 \\
letter                          & 80.3300 & 85.4300 & 80.0900 & 84.7800 & 84.7350 & 87.9345 \\
low\_res\_spect                 & 87.7588 & 89.2647 & 87.7605 & 89.0760 & 86.2511 & 88.3560 \\
lung\_cancer                    & 50.9524 & 66.1905 & 53.8095 & 53.3333 & 71.9048 & 72.3810 \\
lymphography                    & 86.4828 & 88.4598 & 85.8161 & 88.4828 & 85.7701 & 86.0575 \\
molec\_biol\_splice             & 51.8809 & 58.0251 & 51.8809 & 57.0219 & 68.0564 & 69.4679 \\
nursery                         & 70.5324 & 75.2546 & 70.2623 & 73.7809 & 70.7022 & 73.8426 \\
optical                         & 96.2100 & 98.1139 & 96.4769 & 97.9181 & 96.6370 & 97.1174 \\
page\_blocks                    & 95.2491 & 95.7974 & 95.2490 & 95.6696 & 95.3771 & 96.7607 \\
pendigits                       & 98.4443 & 99.0448 & 98.4807 & 98.9902 & 98.9811 & 99.1630 \\
pittsburg\_bridges\_MATERIAL    & 75.8442 & 81.5152 & 77.8788 & 82.4675 & 73.9827 & 75.1602 \\
pittsburg\_bridges\_TYPE        & 42.8571 & 49.5238 & 42.8571 & 50.4762 & 42.8571 & 49.5238 \\
post\_operative                 & 72.2222 & 73.3333 & 72.2222 & 72.2222 & 70.0000 & 70.3333 \\
seeds                           & 87.6190 & 89.5238 & 87.1429 & 89.0476 & 89.5238 & 91.9048 \\
semeion                         & 84.4321 & 91.0227 & 84.5565 & 89.9552 & 87.1925 & 92.1602 \\
soybean                         & 89.4450 & 90.6172 & 88.8622 & 89.5942 & 89.0103 & 88.8643 \\
statlog\_landsat                & 81.5385 & 84.6775 & 81.7249 & 84.5532 & 82.3621 & 84.8562 \\
statlog\_vehicle                & 81.0846 & 83.6881 & 81.3227 & 82.8625 & 81.3241 & 80.8514 \\
teaching                        & 70.7097 & 72.0215 & 70.7097 & 71.3548 & 73.3763 & 72.7097 \\
vertebral\_column\_3clases      & 64.8387 & 65.4839 & 65.1613 & 65.1613 & 67.4194 & 69.6774 \\
wall\_following                 & 77.2369 & 79.8034 & 77.1454 & 80.4447 & 78.1901 & 80.0443 \\
wine                            & 97.2222 & 96.6349 & 95.5238 & 96.0952 & 97.7778 & 97.2222 \\
zoo                             & 95.0000 & 95.0000 & 95.0000 & 95.0000 & 97.0000 & 97.0000 \\
\midrule
\textbf{Average Accuracy}       & \textbf{73.33} & \textbf{75.94} & \textbf{73.34} & \textbf{75.35} & \textbf{74.96} & \textbf{76.20} \\
\textbf{Average Std.\ Dev.}     & \textbf{10.43} & \textbf{9.11}  & \textbf{10.12} & \textbf{8.99}  & \textbf{7.98}  & \textbf{7.17} \\
\textbf{Average Rank}           & \textbf{1.96}  & \textbf{1.04}  & \textbf{1.94}  & \textbf{1.06}  & \textbf{1.86}  & \textbf{1.14} \\
\bottomrule
\multicolumn{7}{l}{$^{\dagger}$ denotes the proposed counterpart.}
\end{tabular}
}
\end{table*}

\begin{figure*}[t]
\centering
\subfloat[]{
    \includegraphics[width=0.30\textwidth]{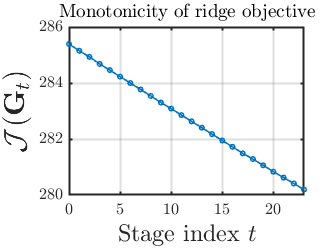}
    \label{fig:mono_objective}
}
\hfill
\subfloat[]{
    \includegraphics[width=0.30\textwidth]{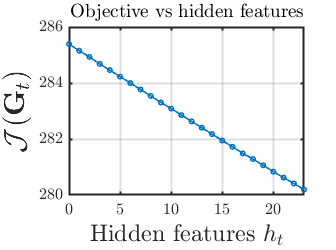}
    \label{fig:objective_vs_hidden}
}
\hfill
\subfloat[]{
    \includegraphics[width=0.30\textwidth]{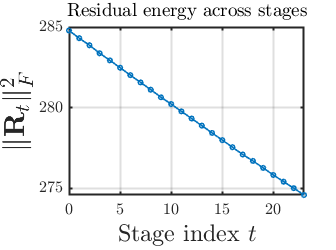}
    \label{fig:residual_energy}
}
\caption{Stagewise behavior of the proposed residual-guided framework on the breast\_cancer dataset, illustrating the monotonic decrease of the ridge-regularized training objective and residual energy as hidden features are incrementally added.}
\label{fig:monotonicity}
\end{figure*}
\section{Empirical Illustration of Monotonicity}
To empirically validate the monotonicity property established in Section~IV.C of the main script, we analyze the stagewise behavior of the proposed residual-guided framework. We consider the breast\_cancer dataset and track the evolution of the ridge-regularized training objective $\mathcal{J}(\mathbf{G}_t)$, the number of selected hidden features $h_t$, and the residual energy $\|\mathbf{R}_t\|_F^2$ across successive stages of the algorithm. Fig.~\ref{fig:monotonicity} illustrates these quantities as the residual-guided construction proceeds. As additional hidden features are incrementally selected based on residual-driven utility, the training objective exhibits a strictly non-increasing trend with respect to both the stage index and the hidden-feature budget. This behavior is consistent with the theoretical guarantee and confirms that each stage makes concrete progress toward reducing the overall objective. Moreover, the residual energy decreases steadily across stages, indicating that newly added features effectively capture previously unexplained components of the supervised signal. These results provide clear empirical evidence that the proposed residual-guided strategy enables stable and consistent optimization behavior, in contrast to one-shot random feature construction.


\end{document}